\documentclass{article}

\usepackage{amssymb,amsmath,amsthm,mathabx}
\renewcommand{\phi}{\varphi}
\newtheorem{claim}{Claim}
\usepackage{times}  
\usepackage{helvet}  
\usepackage{courier}  
\usepackage{url}  
\usepackage{graphicx}  
\newtheorem{definition}{Definition}
\newtheorem{theorem}{Theorem}
\newtheorem{lemma}{Lemma}
\renewcommand{\phi}{\varphi}
\renewcommand{\epsilon}{\varepsilon}

\begin{document}

\title{Epistemic Navigability}
\title{Epistemic Strategies for Navigability}
\title{Navigating with Imperfect Knowledge}
\title{Navigating with Imperfect Information}
\title{Navigation with Imperfect Information}
\title{Navigability with Imperfect Information}

\author{Kaya Deuser and Pavel Naumov\\
Vassar College\\
Poughkeepsie, New York}
\date{}

\maketitle

\begin{abstract}

The article studies navigability of an autonomous agent in a maze where some rooms may be indistinguishable. In a previous work the authors have shown that the properties of navigability in such a setting depend on whether an agent has perfect recall. Navigability by an agent with perfect recall is a transitive relation and without is not transitive. 

This article introduces a notion of restricted navigability and shows that a certain form of transitivity holds for restricted navigability, even for an agent without perfect recall. The main technical result is a sound and complete logical system describing the properties of restricted navigability.
\end{abstract}

\maketitle

\section{Introduction}

Autonomous agents such as  self-navigating missiles, self-driving cars, and robotic vacuum cleaners are often facing the challenge of navigating under conditions of uncertainty about their exact location. A solution to such a problem can be formally described as a sequence of instructions that transition a system from one state to another, assuming that the agent cannot distinguish some of the states. We refer to such systems as {\em epistemic transition systems}.

\begin{figure}[ht]
\begin{center}
\vspace{0mm}
\scalebox{.7}{\includegraphics{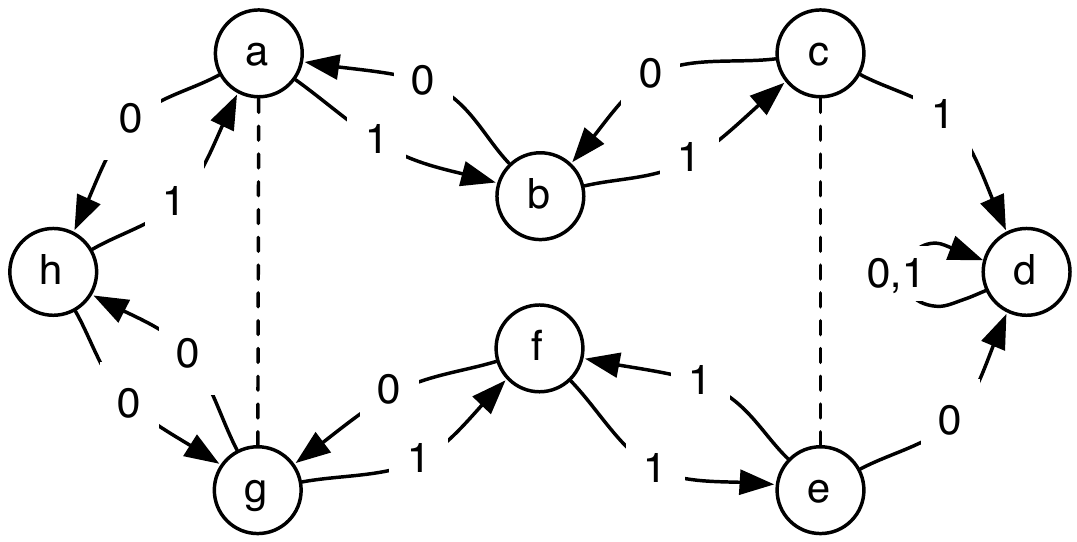}}
\vspace{0mm}
\caption{Epistemic transition system $T_0$.}\label{intro-example figure}
\vspace{0mm}
\end{center}
\vspace{0mm}
\end{figure}
Figure~\ref{intro-example figure} depicts an example of an epistemic transition system $T_0$. This system consists of eight states $a$ through $h$. The agent cannot distinguish state $a$ from state $g$ and state $c$ from state $e$, which is shown in the figure by dashed lines connecting indistinguishable states. The directed arrows in the figure represent the possible transitions that the system can take. The labels on the arrows specify the instructions that the agent needs to use to accomplish this transition. For example, the agent can use instruction 1 to transition the system from state $a$ to state $b$. 

\subsection{Navigability}

The agent can combine multiple instructions into a {\em strategy} to navigate between states that are not directly connected. For example, the agent can apply a strategy that uses instruction 1 repeatedly to navigate from state $h$ to state $d$.

\subsubsection{Amnesic vs Recall Strategies}

We assume that the strategy is ``hardwired" into the agent's read-only memory and cannot be changed once the navigation starts. It is crucial for our discussion to distinguish agents that, in addition to read-only memory, are also equipped with read-write memory. A strategy of the former agent can make the decisions which instruction to use based only on the available information about the current location. A strategy of the latter agent can make such decisions based on the logs of the states the agent previously visited and the instructions the agents used in the past.

An example of these two type of agents are Roomba and Neato, two popular brands of robotic vacuum cleaners. Although Roomba has read-write memory to keep a cleaning schedule, it does not use this memory for navigation. As a result, its behavior is completely determined by the information about the machine's current location: it changes the direction if it hits a wall, it spins if it encounters a dirty spot, etc. On the other hand, Neato scans the room before cleaning and uses this information to navigate. Thus, its strategy is based not only on the information about the current location, but also on the previously obtained information stored in read-write memory.

An agent that has no read-write memory can only use the  available information about the current location while deciding which strategy to use. Any strategy of this agent must use the same instruction in all indistinguishable states. Such a strategy can be formally defined as a function from classes of indistinguishable states into instructions. In this article we refer to such strategies as {\em amnesic} strategies. An example of an amnesic strategy is the algorithm hardwired into Roomba vacuum cleaners by the manufacturer. An ideal agent that has an unlimited size of read-write memory can keep logs of all the states the agent previously visited and all the instructions the agent used in the past. Such an agent is usually called an agent with {\em perfect recall}. In this article we use the term {\em recall strategy} to describe any strategy that potentially can be employed by an agent with perfect recall. Formally, recall strategy is a function that maps indistinguishability classes of {\em histories} into instructions. Although Neato only has a finite memory, the strategy that it uses is a recall strategy.

\begin{figure}[ht]
\begin{center}
\vspace{0mm}
\scalebox{.7}{\includegraphics{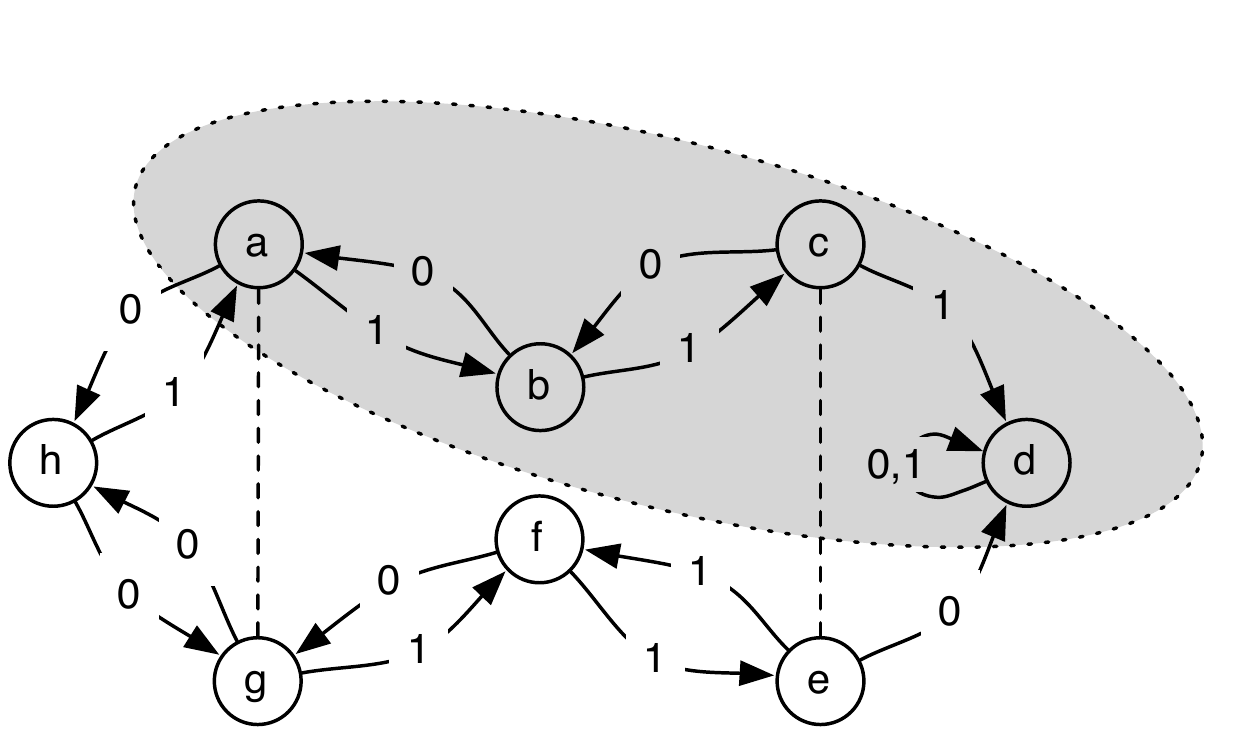}}
\vspace{0mm}
\caption{``Locked" area in epistemic transition system $T_0$.}\label{intro-example-locked figure}
\vspace{0mm}
\end{center}
\vspace{0mm}
\end{figure}
In epistemic transition system $T_0$ depicted in Figure~\ref{intro-example figure}, there is no amnesic strategy to navigate from state $b$ to state $f$. Indeed, any such strategy would have to use the same instruction $i_0$ in indistinguishable states $a$ and $g$. If $i_0=1$, then the agent has to use instruction $1$ in state $a$. Thus, when navigating from state $b$ the system is ``locked" among states $b,a,c$, and $d$, unable to reach states $h,g,f$, and $e$, see  Figure~\ref{intro-example-locked figure}. Similarly, if $i_0=0$, then the agent has to use instruction $0$ in state $g$. Thus, when navigating from state $b$ the system is locked among states $b,a,h,g,c$, and $d$, unable to reach states $e$ and $f$.

However, there is a recall strategy to navigate from state $b$ to state $f$ in system $T_0$. An  example of such a strategy is a strategy that uses instruction $0$ until it visits the class of indistinguishable states $[a]=\{a,g\}$ for {\em the second} time. After that it switches to instruction 1.

\subsubsection{Strategies between Classes}

In epistemic transition system $T_0$, there is a strategy to navigate from state $c$ to state $a$ by using instruction $0$ all the time. However, if the agent is deployed in state $c$, then she does not know that such a strategy exists because she cannot distinguish state $c$ from state $e$, from which instruction $0$ would lock her in state $d$. Additionally, if the agent does not have perfect recall, then even if she decides to use such a strategy to reach state $a$, she would not be able to verify that the goal is accomplished because she cannot distinguish state $a$ from state $g$. For these reasons, in this article we consider navigability not between individual states, but between classes of indistinguishable states. For example, a strategy that uses instruction 1 in all states is an amnesic strategy to navigate from class $[a]=\{a,g\}$ to class $[c]=\{c,e\}$. Indeed, starting from state $a$ such strategy leads to state $c$ and starting from state $g$ this strategy leads to state $e$. 

At the same time, not only is there no amnesic, but there is also no recall strategy to navigate from class $[c]$ to class $[a]$. Indeed, consider two situations when the system {\em starts} in (i) state $c$ and (ii) state $e$. The histories of the transitions in these two cases consist only of states $c$ and $e$ respectively. Thus, these two {\em histories} are indistinguishable and a recall strategy would have to use the same instruction $i_0$ in both settings. If $i_0=1$, then from state $c$ the system transitions into state $d$ and remains locked there. If $i_0=0$, then from state $e$ the system also transitions into state $d$ and remains locked there. In either of these two cases the system is not able to reach a state in class $[a]$.

As one would expect, there are situations when there is a recall strategy, but no amnesic strategy to navigate between two classes. For example, there is a recall strategy to navigate from class $[a]$ to class $[b]$. Perhaps unexpectedly, this strategy always uses instruction 0 on the first transition. Then, it switches to repeatedly using instruction 1. Starting from either state $a$ or state $g$, this strategy first transitions the system into state $g$, then into state $a$, then into state $b$. Thus, it is a recall strategy to navigate from class $[a]$ to class $[b]$. To show that there is no amnesic strategy to navigate from class $[a]$ to class $[b]$, note that any such strategy would have to use the same instruction $i_0$ in state $a$ and state $g$. If $i_0=0$ and the navigation starts from state $a$, then the system is locked among states $a,h$, and $g$ and never reaches states $b,c,d,e$, and $f$. If $i_0=1$ and the navigation starts from state $g$, then the system is locked among states $g,f,e$, and $d$ and never reaches states $h,a,b$, and $c$.

\begin{table}[ht]
\centering
\begin{tabular}{ c | c c c c c c}
            & $[a]$ & $[b]$ & $[c]$ & $[d]$ & $[f]$ & $[h]$\\ \hline
$[a]$       & a     & r     & a     & r     & r     & a \\
$[b]$       & a     & a     & a     & a     & r     & a \\
$[c]$       & -     & -     & a     & r     & -     & - \\
$[d]$       & -     & -     & -     & a     & -     & -\\
$[f]$       & a     & r     & a     & a     & a     & a\\
$[h]$       & a     & a     & a     & a     & a     & a
\end{tabular}
\caption{Navigability between classes in system $T_0$.}\label{s r table}
\end{table}

Table~\ref{s r table} shows what kind of strategies exist between classes of states in the epistemic transition system $T_0$. Letter ``a" at the intersection of row $x$ and column $y$ marks the cases when there is an amnesic strategy from class $x$ to class $y$. Letter ``r" denotes the cases where there are recall but no amnesic strategies. Dash ``-" marks the cases where neither amnesic nor recall strategies exist. 

An interesting observation about Table~\ref{s r table} is that there is a recall but no amnesic strategy to navigate from class $[c]$ to class $[d]$. An example of such a recall strategy is a strategy that first uses instruction 0 and then repeatedly uses instruction 1. If this recall strategy is used starting from state $e$, then the system transitions directly into state $d$. If the system starts from state $c$, then the system first transitions to state $b$ using instruction 0, then to back to state $c$ using instruction 1, and finally from state $c$ to state $d$ this time using instruction 1. 

We now show that there is no amnesic strategy $s$ to navigate from class $[c]$ to class $[d]$. Suppose the opposite. By $s[a]$ and $s[c]$ we denote the instructions used by strategy $s$ in all states of classes $[a]$ and $[c]$ respectively. Since both $s[a]$ and $s[c]$ can have either value 0 or value 1, there are four cases to consider. If $s[a]=s[c]=0$, then strategy $s$ can be used to navigate from state $e$ to state $d$. However, if the navigation starts in state $c$, then the system is locked among state  $c$, $b$, $a$, $h$, and $g$ and it never reaches states $f$, $e$, and $d$.  If $s[a]=1$ and $s[c]=0$, then when the system starts from state $c$ it is locked among states $c$, $b$, and $a$. The other two cases are similar.

\subsubsection{Strategies between Sets of Classes}

As we have seen above, there is a recall strategy, but no amnesic strategy to navigate from class  $[a]$ to class $[b]$. One can similarly show that there is a recall strategy, but no amnesic strategy to navigate from class  $[a]$ to class $[f]$. However, if the goal is to navigate from class $[a]$ to either class $[b]$ or class $[f]$, then there is an amnesic strategy to achieve this. An example of such a strategy is the strategy that uses instruction 1 in all states. This strategy directly transitions the system from any state in class $[a]$ to a state in either class $[b]$ or class $[f]$. This example shows that navigability to a {\em set of classes} cannot be reduced to navigability to {\em classes}.

\begin{figure}[ht]
\begin{center}
\vspace{0mm}
\scalebox{.7}{\includegraphics{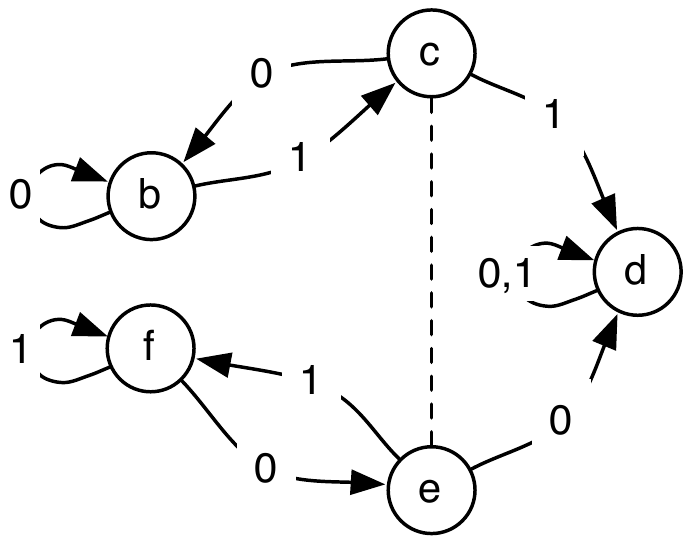}}
\vspace{0mm}
\caption{Epistemic transition system $T_1$.}\label{intro-example-rhs figure}
\vspace{0mm}
\end{center}
\vspace{0mm}
\end{figure}

Our next example shows that navigability {\em from} a sets of classes cannot be reduced to navigability {\em from} classes. Consider epistemic transition system $T_1$ depicted in Figure~\ref{intro-example-rhs figure}. Amnesic strategy that always invokes instruction 1 can be used to transition this system from class $[b]$ to class $[d]$. At the same time, amnesic strategy that always invokes instruction 0 can be used to transition this system from class $[f]$ to class $[d]$. However, there is no single amnesic strategy that would transition the system from an arbitrary state of an arbitrary class of set $\{[b],[f]\}$ to a state in a class of set $\{[d]\}$. Indeed, any such strategy would have to use the same instruction $i_0$ in indistinguishable states $c$ and $e$. If $i_0=0$, then from state $b$ the system cannot reach state $d$ because it is locked among states $b$ and $c$. If $i_0=1$, then from state $e$ the system cannot reach state $d$ because it is locked among states $f$ and $e$. The same example can be made using classes $[b]$, $[f]$, and $[d]$ in the original epistemic transition system $T_0$, but the proof that an amnesic strategy does not exist is more involved. 

The above two examples, show that navigability between sets of classes cannot be expressed in terms of navigability between classes. In this article, we study this more general notion of navigability between sets of classes. If $A$ and $B$ are two sets of classes, then we use notation $A\rhd B$ to denote the existence of a strategy to navigate from set $A$ to set $B$. Whether we mean an amnesic or a recall strategy will be clear from the context.

\subsubsection{Universal properties}

So far, we have discussed properties of navigability specific to epistemic transition system $T_0$. In this article we study properties of navigability that are universally true in all epistemic transition systems. Two examples of such properties are
\begin{enumerate}
\item Reflexivity: $A\rhd B$, where $A\subseteq B$, and
\item Augmentation: $A\rhd B\to (A\cup C)\rhd (B\cup C)$.
\end{enumerate}
The first of these properties claims that there is always a navigation strategy to navigate from a set to a superset of the set. Any strategy can be used to do this because the goal is achieved before the navigation even starts. The second property says that if there is a strategy to navigate from set $A$ to set $B$, then there is a strategy to navigate from set $A\cup C$ to set $B\cup C$. This property is true because {\em the very same strategy} that is used to navigate from set $A$ to set $B$ can {\em also} be used to navigate from set $A\cup C$ to set $B\cup C$. Both of the above properties are universally true for amnesic as well as for recall strategies. An example of a property universally true for recall strategies but not universally true for amnesic strategies is 
\begin{enumerate}
\item[3.] Transitivity: $A\rhd B\to (B\rhd C\to A\rhd C)$.
\end{enumerate}
This property says that if there is a navigation strategy from set $A$ to set $B$ and another navigation strategy from set $B$ to set $C$, then there is a navigation strategy from set $A$ to set $C$. To see that this property is not universally true for amnesic strategies, notice that in epistemic transition system $T_0$ the agent can transition the system from set $\{[a]\}$ to set $\{[h]\}$ using the amnesic strategy that always invokes instruction 0. She can also transition the system from set $\{[h]\}$ to set $\{[b]\}$ using the amnesic strategy that always invokes instruction 1. Yet, as we have seen earlier, there is no amnesic strategy to navigate from set $\{[a]\}$ to set $\{[b]\}$. In our previous work~\cite{dn17arxiv-armstrong}, we have shown that the Transitivity principle is universally true for recall strategies. In the same paper we have also shown that the Reflexivity, the Augmentation, and the Transitivity principles give a complete axiomatization of all universal properties of recall strategies. These three axioms are known in the database literature as Armstrong axioms~\cite[p.~81]{guw09}. They give a sound and complete axiomatization of the properties of functional dependency relation in databases~\cite{a74}. 

As we have seen, the Transitivity principle is not universally true for amnesic strategies. In~\cite{dn17arxiv-armstrong} we gave a complete axiomatization of all properties of amnesic strategies. This axiomatization, in addition to the Reflexivity principle and the Augmentation principle, also contains 
\begin{enumerate}
\item[4.] Left Monotonicity: $A'\rhd B\to A\rhd B$, where $A\subseteq A'$, and
\item[5.] Right Monotonicity: $A\rhd B\to A\rhd B'$, where $B\subseteq B'$.
\end{enumerate}
It can be shown that the Left and the Right Monotonicity principles are provable from Armstrong axioms.
Thus, the single principle that distinguishes universal properties of recall strategies from the universal properties of amnesic strategies is Armstrong's Transitivity axiom. Perhaps unexpectedly, this single principle captures the additional properties of navigability associated with the assumption of perfect recall by the agent.

\subsubsection{Can Transitivity be Saved?}

Transitivity is an intuitively expected property for navigability to have. To us, the fact that navigability under amnesic strategies does not have this property is surprising and counter-intuitive. In the current article we answer the question if there is some form of transitivity that still can be claimed without the assumption of total recall. 

Our answer to this question is that a form of transitivity is true for amnesic strategies if, instead of the navigability relation, one considers the {\em restricted} navigability.

\subsection{Restricted Navigability}

As we have seen earlier, there is an amnesic strategy to navigate from class $[a]$ to class $[c]$ in epistemic transition system $T_0$ depicted in Figure~\ref{intro-example figure}. An example of such a strategy is an amnesic strategy that uses instruction $1$ in each state. Imagine now that the system is restricted from transitioning through class $[f]$. There is still a strategy to navigate from class $[a]$ to class $[c]$ under this restriction, but transitioning from state $g$ to class $[c]$ will now have to go through states $h$, $a$, and $b$. Any strategy that navigates from class $[a]$ to class $[c]$ and avoids class $[f]$ would need to be a recall strategy because it must use instruction 1 in state $a$ and instruction 0 in state $g$. 

In this article we study the universal properties of the restricted navigability relation $A\rhd_B C$, where $A$, $B$, and $C$ are sets of classes. Informally, $A\rhd_B C$ means that there is a strategy to navigate from set $A$ to set $C$ while staying in set $B$. Formal semantics of relation $A\rhd_B C$ is given in Definition~\ref{sat}. For now let us just note that the restriction ``staying in set $B$"  does not apply to destination states. That is, the final point of a navigation path can be outside of set $B$. This is a technical detail that makes the axioms of our logical system more elegant. One can think about the subscript as the set where the agent actually uses  the strategy to determine the instruction. Once the set  $C$ is reached, no further instructions are needed. 

The following form of the Transitivity principle is valid for the restricted navigability under the amnesic strategies:
\begin{equation}\label{new transitivity}
    A\rhd_B C \to (C\rhd_D  E \to A \rhd_{B\cup D} E), \mbox{ where $B\cap D=\varnothing$.}
\end{equation}
We prove the soundness of this principle in Lemma~\ref{transitivity sound}.

\subsubsection{Universal Principles}

In this article we give a sound and complete axiomatization of the universal properties of relation $A\rhd_B C$. This axiomatization contains a modified version of Armstrong's Reflexivity principle:
\begin{equation}\label{new reflexivity}
    A\rhd_B C, \mbox{ where $A\subseteq C$},
\end{equation}
a modified version of Armstrong's Augmentation principle:
\begin{equation}\label{new augmentation}
    A\rhd_B C \to (A\cup D)\rhd_B (C\cup D),
\end{equation}
and the modified version~(\ref{new transitivity}) of Armstrong's Transitivity principle. Note that transitivity principle~(\ref{new transitivity}) can be stated in a more general form where the restriction $B\cap D=\varnothing$ is replaced with $B\cap D\subseteq C$. We prove this seemingly more general form of transitivity from principle~(\ref{new transitivity}) and the rest of the axioms of our logical system in Lemma~\ref{super transitivity}.

In addition to the modified versions of three Armstrong's axioms described above, our logical system contains three more axioms. The first of these axioms is 
\begin{equation}\label{early bird}
    A\rhd_B C \to A\rhd_{B\setminus C} C.
\end{equation}
This axiom says that if there is a strategy to navigate from set $A$ to set $C$ passing only through set $B$, then there is a strategy to navigate from set $A$ to set $C$ passing only through set $B\setminus C$. Indeed, once a state of a class of the set $C$ is reached for the first time, there is no need to continue the execution of the strategy. Thus, the navigation path will never have to pass through set $B\cap C$. In essence, the axiom states that the execution of any strategy can be terminated once the desired goal is reached for the first time. For this reason we refer to principle~(\ref{early bird}) as the Early Bird principle.

Before we state the remaining two axioms of our logical system, we need to discuss the meaning of the statement $A\rhd_\varnothing\varnothing$. It claims that there is a strategy to navigate from each state in each class of set $A$ to a state in an empty set of classes. This is only possible if all classes in set $A$ are empty.  Although in mathematics it is common to consider only nonempty equivalence classes, in this article we allow for empty classes as well. For example, when navigating in a maze one might consider a strategy that  takes a left door in all red rooms, even if the class of red rooms is empty. In fact, we allow multiple empty classes: an empty class of red rooms, an empty class of blue rooms, etc. In different empty classes the strategy might use different instructions. The technical details of our semantics are described in Section~\ref{semantics section}. For now let us just point out that statement $A\rhd_\varnothing\varnothing$ means that all classes in set $A$ are empty.

We are now ready to state the two remaining axioms of our logical system. One of them is
\begin{equation}\label{set minus axiom}
    A\rhd_\varnothing B \to A\setminus B\rhd_\varnothing\varnothing.
\end{equation}
It states that if there is a strategy to navigate from each state in each class in set $A$ to a state in a class of set $B$ {\em through an empty set of classes}, then all starting states in classes from set $A$ must already be in classes of set $B$. In other words, classes in set $A\setminus B$ must be empty.  We call this the Trivial Path principle.

The last axiom of our logical system is
\begin{equation}\label{empty set axiom}
    A\rhd_B \varnothing  \to A\rhd_\varnothing\varnothing.
\end{equation}
It states that if there is a strategy to navigate from each state in each class in set $A$ to a state in a class of an empty set, then classes in set $A$ must be empty. We call this the Path to Nowhere principle.

Our main results are the soundness and completeness theorems for the logical system consisting of axioms  (\ref{new transitivity}), (\ref{new reflexivity}), (\ref{new augmentation}), (\ref{early bird}), (\ref{set minus axiom}), and (\ref{empty set axiom}). 

\subsection{Literature Review}

Most of the existing literature on logical systems for reasoning about strategies is focused on modal logics for coalition strategies. Logics of coalition power were proposed by Marc Pauly~\cite{p01illc,p02}, who also proved the completeness of the basic logic of coalition power. Pauly's approach has been widely studied in literature~\cite{g01tark,vw05ai,b07ijcai,sgvw06aamas,abvs10jal,avw09ai,b14sr}. An alternative, binary-modality-based, logical system  was proposed by More and Naumov~\cite{mn12tocl}. 

Alur, Henzinger, and Kupferman introduced the Alternating-Time Temporal Logic (ATL) that combines temporal and coalition modalities~\cite{ahk02}.
Van der Hoek and Wool\-dridge proposed to combine ATL with epistemic modality to form Alternating-Time Temporal Epistemic Logic~\cite{vw03sl}. However, they did not prove the completeness theorem for the proposed logical system. 
A completeness theorem for a logical system that combines coalition power and epistemic modalities was proven by {\AA}gotnes and Alechina~\cite{aa12aamas}. 

The notion of a strategy that we consider in this article is much more restrictive than the the notion of strategy in the works mentioned above. Namely, we assume that the strategy must be based only on the information available to the agent. This is captured in our setting by requiring the strategy to be the same in all indistinguishable states or all indistinguishable histories.  This restriction on strategies has been studied before under different names. Jamroga and {\AA}gotnes talk about ``knowledge to identify and execute a strategy"~\cite{ja07jancl},  Jamroga and van der Hoek discuss ``difference between an agent knowing that he has a suitable strategy and knowing the strategy itself"~\cite{jv04fm}. Van Benthem calls such strategies ``uniform"~\cite{v01ber}. Naumov and Tao~\cite{nt17aamas} used the term ``executable strategy". 
{\AA}gotnes and Alechina gave a complete axiomatization of an interplay between single-agent knowledge and know-how modalities~\cite{aa16jlc}. Naumov and Tao~\cite{nt17aamas} axiomatized interplay between distributed knowledge modality and know-how coalition strategies for enforcing a condition indefinitely. A similar complete logical system in a {\em single-agent} setting for know-how strategies to {achieve} a goal in multiple steps rather than to {maintain} a goal is developed by Fervari, Herzig, Li, Wang~\cite{fhlw17ijcai}. Naumov and Tao proposed a complete trimodal logical system that describes an interplay between distributed knowledge, coalition power, and know-how coalition power modalities for goals achievable in one steps~\cite{nt17tark}. All of the above works do not consider a perfect recall setting. Modal logic that combines distributed knowledge with coalition power in a perfect recall setting has been recently proposed by Naumov and Tao~\cite{nt17arxiv-perfect}. Unlike this article, they only consider goals achievable in one step.


Our work is closely connected to Wang's logic of knowing how~\cite{w15lori,w17synthese} and especially to Li and Wang's logic of knowing how with intermediate constraints~\cite{lw17icla}. The latter logical system describes the properties of the ternary modality ${\sf Khm}(a,b,c)$. The modality stands for ``there is a strategy to navigate from $a$ to $b$ while passing only through $c$". Although in \cite{lw17icla} variables $a$, $b$, and $c$ stand for statements, not sets of classes, like in our article, this difference is not very significant. The significant difference between Li and Wang's approach and ours is how we define strategy. They define strategy as a sequence of instructions that, when executed in the given order, lead from one state to another. For example, according to their definition, there is a strategy to navigate from state $c$ to state $a$ of the epistemic transition system depicted in Figure~\ref{intro-example figure}. The strategy is the sequence $00$. In other words, unlike our approach, their notion of strategy does not allow the agent to change her behaviour based on her current observations. In terms of our motivation example, their agent is a robot without sensors equipped with a stack that stores the list of instructions to be executed. This difference in semantics results in very different logical systems. For example, Armstrong's Transitivity principle, rephrased in their language, is valid under their semantics. 

Our current work builds on our previous paper~\cite{dn17arxiv-armstrong}, where we have shown that Armstrong axioms are sound and complete with respect to (unrestricted) navigability with recall strategies. In the same work we proved that logical system consisting of the Reflexivity, the Augmentation, the Left Monotonicity, and the Right Monotonicity axioms is  sound and complete with respect to (unrestricted) navigability with amnesic strategies. Formally, the main contribution of this work is a sound and complete logical system for {\em restricted} navigability under amnesic strategies. Informally, our contribution is the observation that there is a middle ground setting in which a form of transitivity holds even for agents without perfect recall.

\subsection{Article Outline}

The rest of this article is structured as following. In Section~\ref{syntax section} we introduce the syntax of our logical system. In Section~\ref{semantics section} we define the formal semantics of this system. Section~\ref{axioms section} lists the axioms of our system. In Section~\ref{examples section} we give several examples of formal derivations. These results will be used later in the proof of completeness. 
Section~\ref{soundness section} and Section~\ref{completeness section} prove the soundness and the completeness of our logical system respectively. Section~\ref{conclusion section} concludes. 

\section{Syntax}\label{syntax section}

In the introduction we discussed the ternary relation $A\rhd_B C$ as a relation between three sets of indistinguishability classes of a given epistemic transition system. In the rest of this article we present a formal logical system capturing the universally true properties of this relation in an arbitrary epistemic transition system. Note that sets of classes $A$, $B$ and $C$ are specific to a particular epistemic transition system and thus can not be used to formally state properties common to all epistemic transition systems. In this article we overcome this issue by assuming that there is a fixed {\em finite} set $V$ of {\em views}. A transition model specifies an observation function from states to views instead of specifying an indistinguishability relation between states. Informally, two states are indistinguishable if and only if the values of the observation function on these two states  are equal. Using views instead of classes the language of our logical system can be defined independently from a particular epistemic transition system as follows.

\begin{definition}
Let $\Phi$ be the minimal set of formulae such that
\begin{enumerate}
    \item $A\rhd_B C\in\Phi$ for each sets $A,B,C\subseteq V$,
    \item $\neg\phi,\phi\to\psi\in \Phi$, for each $\phi,\psi\in\Phi$.
\end{enumerate}
\end{definition}

The above assumption that set of views $V$ is finite is used later in the proof of completeness.

\section{Semantics}\label{semantics section}

In this section we define formal semantics of our logical system.

\begin{definition}\label{transition system}
Epistemic transition system is a tuple
$(S,o,I,\{\to_i\}_{i\in I})$, where
\begin{enumerate}
    \item $S$ is a set of states,
    \item $o: S\to V$ is an observation function,
    \item $I$ is a set of instructions,
    \item $\to_i$ is a binary relation between states for each $i\in I$.
\end{enumerate}
\end{definition}

For example, for the epistemic transition system $T_0$ depicted in Figure~\ref{intro-example figure}, set $S$ is $\{a,b,c,d,e,f,g,h\}$. Observation function $o$ is such that  $o(a)=o(g)=v_1$, $o(b)=v_2$, $o(c)=o(e)=v_3$, $o(d)=v_4$, $o(f)=v_5$ and $o(h)=v_6$, where $v_1$, $v_2$, $v_3$, $v_4$, $v_5$, and $v_6$ are arbitrary distinctive elements of set $V$. Instruction set $I$ is $\{0,1\}$. Relation $\to_0$ is $\{(a,h),(b,a),(c,b),(d,d),(e,d),(f,g),(g,h)$, $(h,g)\}$ and relation $\to_1$ is $\{(a,b),(b,c),(c,d),(d,d),(e,f),(f,e),(g,f),(h,a)\}$.

Note that epistemic transition system $T_0$ is deterministic in the sense that there is a unique state $v$ into which the system transitions from a given state $w$  under a given instruction $i$. Definition~\ref{transition system} specifies a transition system in terms of a transitive relation $\rightarrow_i$. We allow several states $v$ such that $w\to_i v$ to model {\em nondeterministic} transitions. We allow the set of such states $v$ to be empty to model {\em terminating} transitions. Informally, if in a state $w$ an instruction $i$ is invoked such that there is no $v$ for which $w\to_i v$, then the system terminates and no further instructions are executed. 

In the introduction we made a distinction between amnesic and recall strategies. Since the rest of the article deals only with amnesic strategies, we refer to such strategies simply as {\em strategies}.

\begin{definition}\label{strategy}
A strategy is an arbitrary function from $V$ to $I$.
\end{definition}

A side effect of our choice to use views instead of equivalence classes is that there might be one or more views that are not values of the observation function on any state in the epistemic transition system. Such views define what we call ``empty equivalence classes" in the introduction. Per Definition~\ref{strategy}, a strategy must still be defined on such views.

The next definition specifies the paths in an epistemic transition system that start in a given set of views and are compatible with a given strategy.

\begin{definition}\label{path}
For any set $A\subseteq V$ and any strategy $s$, let $Path_s(A)$ be the set of all (finite or infinite) sequences $w_0,w_1,w_2,\dots\in S$ such that 
\begin{enumerate}
    \item $o(w_0)\in A$, and
    \item $w_k\to_{s(o(w_k))}w_{k+1}$ for each $k\ge 0$ for which $w_{k+1}$ exists.
\end{enumerate}
\end{definition}
For epistemic transition system $T_0$, sequence $a,b,c$ belongs to set $Path_s(\{o(a)\})$, where $o(a)$ is the view assigned in system in $T_0$ to state $a$ and $s$ is a strategy such that $s(v)=1$ for each $v\in V$.

A path $\pi'$ is an extension of a path $\pi$  if  $\pi$ is a prefix of $\pi'$ or $\pi'=\pi$.


\begin{definition}\label{maxpath}
$MaxPath_s(A)$ is the set of all sequences in $Path_s(A)$ that are either infinite or cannot be extended to a longer sequence in $Path_s(A)$.
\end{definition}
For epistemic transition system $T_0$, the infinite sequence $a,b,c,d,d,\dots$ belongs to the set $MaxPath_s(\{o(a)\})$, where $o(a)$ is the view assigned in that epistemic transition system to state $a$ and $s$ is a strategy such that $s(v)=1$ for each $v\in V$.

\begin{lemma}\label{maxpath exists}
Any sequence that is in set $Path_s(A)$ can be extended to a sequence in set $MaxPath_s(A)$. 
\end{lemma}
\begin{proof}
Any sequence in $Path_s(A)$ which is not in $MaxPath_s(A)$ can be extended to a longer sequence in $Path_s(A)$. Repeating this step multiple times one can get a finite or an infinite sequence in $MaxPath_s(A)$.
\end{proof}
The next definition introduces a technical notation that we use to define the semantics of the restricted navigability relation $A\rhd_B C$.
\begin{definition}\label{until}
Let $Until(A,B)$ be the set of all such sequences $w_0,w_1,w_2,\dots$ that there is $k_0\ge 0$ where
\begin{enumerate}
    \item $o(w_k)\in A$ for each $k< k_0$ and
    \item $o(w_{k_0})\in B$.
\end{enumerate}
\end{definition}

\begin{definition}\label{sat}
For any epistemic transition system $T$ and any formula $\phi$, satisfiability relation $T\vDash\phi$ is defined as follows:
\begin{enumerate}
    \item $T\vDash A\rhd_B C$ if $MaxPath_s(A)\subseteq Until(B,C)$ for some strategy $s$,
    \item $T\vDash\neg\phi$ if $T\nvDash\phi$,
    \item $T\vDash\phi\to\psi$ if $T\nvDash\phi$ or $T\vDash\psi$.
\end{enumerate}

\end{definition}

\section{Axioms}\label{axioms section}
In addition to the propositional tautologies in language $\Phi$, our logical system contains the following axioms:
\begin{enumerate}
    \item Reflexivity: $A \rhd_B C$, where $A \subseteq C$,
    \item Augmentation: $A \rhd_B C\to (A\cup D) \rhd_B (C\cup D)$,
    \item Transitivity: $A \rhd_B C \to (C \rhd_D E\to A\rhd_{B\cup D} E)$, where $B\cap D=\varnothing$,
    \item Early Bird: $A \rhd_B C\to A \rhd_{B\setminus C} C$,
    \item Trivial Path: $A\rhd_\varnothing B\to (A\setminus B)\rhd_\varnothing \varnothing$,
    \item Path to Nowhere: $A\rhd_B\varnothing\to A\rhd_\varnothing\varnothing$.
\end{enumerate}
We write $\vdash\phi$ if formula $\phi$ is provable in our logical system using the Modus Ponens inference rule. We write $X\vdash\phi$ if formula $\phi$ is provable from the axioms of our system and the set of additional axioms $X$.

\section{Examples of Derivations}\label{examples section}

The soundness and the completeness of our logical system is proven in Section~\ref{soundness section} and Section~\ref{completeness section}. In this section we give several examples of formal derivations in this system. The results obtained here are used in the proof of the completeness.

\begin{lemma}\label{remove left}
$\vdash A\rhd_B C\to A'\rhd_B C$, where $A'\subseteq A$. 
\end{lemma}
\begin{proof}
 Assumption $A'\subseteq A$ implies that $\vdash A'\rhd_\varnothing A$ by the Reflexivity axiom. Hence, $\vdash A\rhd_B C\to A'\rhd_{\varnothing\cup B} C$ by the Transitivity axiom. Therefore, $\vdash A\rhd_B C\to A'\rhd_{B} C$.
\end{proof}

\begin{lemma}\label{add down}
$\vdash A\rhd_B C\to A\rhd_{B'} C$, where $B\subseteq B'$. 
\end{lemma}
\begin{proof}
By the Reflexivity axiom, $\vdash A\rhd_{B'\setminus B} A$. Hence, by the Transitivity axiom, $\vdash A\rhd_B C\to A\rhd_{(B'\setminus B)\cup B} C$. Therefore, $\vdash A\rhd_B C\to A\rhd_{B'} C$, due to the assumption $B\subseteq B'$.
\end{proof}

\begin{lemma}\label{add right}
$\vdash A\rhd_B C\to A\rhd_{B} C'$, where $C\subseteq C'$. 
\end{lemma}
\begin{proof}
By the Transitivity axiom, $\vdash A\rhd_B C\to(C\rhd_\varnothing C'\to A\rhd_B C')$. Hence, by the laws of logical reasoning, 
\begin{equation}\label{trans c}
\vdash C\rhd_\varnothing C'\to (A\rhd_B C\to A\rhd_B C').
\end{equation}
At the same time, the assumption $C\subseteq C'$ implies $\vdash C\rhd_\varnothing C'$ by the Reflexivity axiom. Hence, $A\rhd_B C\to A\rhd_B C'$ due to statement (\ref{trans c}).
\end{proof}

\begin{lemma}\label{remove void}
$\vdash B'\rhd_\varnothing\varnothing \to(A\rhd_B C \to A\rhd_{B\setminus B'} C).$
\end{lemma}
\begin{proof}
First, $\vdash A\rhd_B C\to A\rhd_{B}(C\cup B')$ by Lemma~\ref{add right}. Second, the formula $A\rhd_B (C\cup B')\to A \rhd_{B\setminus (C\cup B')}(C\cup B')$ is an instance of the Early Bird axiom. Third, $\vdash A\rhd_{B\setminus (C\cup B')}(C\cup B')\to A\rhd_{B\setminus B'}(C\cup B')$ by Lemma~\ref{add down} because $B\setminus (C\cup B')\subseteq B\setminus B'$. The three statements above by the laws of propositional reasoning imply that 
\begin{equation}\label{big formula}
\vdash A\rhd_B C \to A\rhd_{B\setminus B'}(C\cup B').
\end{equation}
At the same time, 
\begin{equation}\label{small formula}
B'\rhd_\varnothing\varnothing\to (B'\cup C)\rhd_\varnothing C
\end{equation}
is an instance of the Augmentation axiom. Finally, the formula
\begin{equation}\label{very big formula}
A\rhd_{B\setminus B'}(C\cup B')\to((B'\cup C)\rhd_\varnothing C \to A\rhd_{B\setminus B'}C)
\end{equation}
is an instance of the Transitivity axiom. Taken together, statement~(\ref{big formula}), statement~(\ref{small formula}), and statement~(\ref{very big formula}) imply by the laws of propositional reasoning that
$\vdash B'\rhd_\varnothing\varnothing \to(A\rhd_B C \to A\rhd_{B\setminus B'} C)$.
\end{proof}

\begin{lemma}\label{super transitivity}
$\vdash A\rhd_B C \to (C \rhd_D E \to A \rhd_{B\cup D} E)$ if $B\cap D\subseteq C$.
\end{lemma}
\begin{proof}
Assumption $B\cap D\subseteq C$ implies that $B\setminus C\subseteq B\setminus D$. Thus, by Lemma~\ref{add down}, $\vdash A\rhd_{B\setminus C}C\to A\rhd_{B\setminus D} C$.
At the same time $\vdash A\rhd_B C \to A\rhd_{B\setminus C} C$ by the Early Bird axiom. Hence, by the laws of propositional reasoning, 
$$\vdash A\rhd_B C\to  A\rhd_{B\setminus D} C.$$ 
Note that $A\rhd_{B\setminus D}C\to(C\rhd_D E \to A\rhd_{B\cup D} E)$ is an instance of the Transitivity axiom. Therefore, $\vdash A\rhd_B C \to (C \rhd_D E \to A \rhd_{B\cup D} E)$ by the laws of propositional reasoning.
\end{proof}

\section{Soundness}\label{soundness section}

In this section we prove the soundness of our logical system. The soundness of each of the axioms is stated as a separate lemma. The soundness theorem for the logical system is given in the end of the section. We start with a technical lemma that lists properties of sets $Until$ and $MaxPath$. These properties are used in the proofs of the soundness of the axioms. 

\begin{lemma}\label{until lemma} 
For any set $A,B,C\subseteq V$ and any strategy $s$:
\begin{enumerate}
    \item $MaxPath_s(A)\subseteq Until(\varnothing,A)$,
    \item $MaxPath_s(A)\subseteq MaxPath_s(A')$, where $A\subseteq A'$,
    \item $MaxPath_s(A\cup B) = MaxPath_s(A) \cup MaxPath_s(B)$,
    \item $Until(A,\varnothing)=\varnothing$,
    \item $Until(A,B)\subseteq Until(A,B')$, where $B\subseteq B'$,
    \item $Until(A,B)\subseteq Until(A',B)$, where $A\subseteq A'$,
    \item $Until(A,B)\subseteq Until(A\setminus B,B)$,
    \item $Until(\varnothing,A)\cap Until(\varnothing, B)\subseteq Until(\varnothing,A\cap B)$.
\end{enumerate}
\end{lemma}
\begin{proof}
Statements 1 through 8 follow from Definition~\ref{maxpath} and Definition~\ref{until}.
\end{proof}
Next, we show the soundness of the Reflexivity and Augmentation axioms.
\begin{lemma}
If $A\subseteq C$, then $T\vDash A\rhd_B C$.
\end{lemma}
\begin{proof}
By Lemma~\ref{until lemma} and the assumption $A\subseteq C$,
$$
MaxPath_s(A)\subseteq Until(\varnothing,A)\subseteq Until(\varnothing,C)\subseteq Until(B,C).
$$
Therefore, $T\vDash A\rhd_B C$ by Definition~\ref{sat}.
\end{proof}

\begin{lemma}
If $T\vDash A\rhd_B C$, then $T\vDash (A\cup D)\rhd_B (C\cup D)$.
\end{lemma}
\begin{proof}
Let $T\vDash A\rhd_B C$.
Thus, $MaxPath_s(A)\subseteq Until(B,C)$ by Definition~\ref{sat}. Hence, by  Lemma~\ref{until lemma},
\begin{eqnarray*}
MaxPath_s(A\cup D)&=& MaxPath_s(A) \cup MaxPath_s(D)\\
&\subseteq& Until(B,C)\cup Until(\varnothing,D)\\
&\subseteq& Until(B,C)\cup Until(B,D)\\
&\subseteq& Until(B,C\cup D)\cup Until(B,C\cup D)\\
&=&Until(B,C\cup D).
\end{eqnarray*}
Therefore, $T\vDash A,D\rhd_B C,D$ by Definition~\ref{sat}.
\end{proof}

The proof of the soundness of the Transitivity axiom is based on the following auxiliary lemma.

\begin{lemma}\label{switch s}
If $MaxPath_{s_1}(A)\subseteq Until(B,C)$ and  $s_1(v)=s_2(v)$ for each $v\in B$, then $MaxPath_{s_2}(A)\subseteq Until(B,C)$.
\end{lemma}
\begin{proof}
Consider any sequence $\pi=w_0,w_1,\dots\in MaxPath_{s_2}(A)$. We prove that $\pi\in Until(B,C)$ by separating the following two cases:

\noindent{\em Case I}: $o(w_k)\in B\setminus C$ for each $k\ge 0$. Thus, $s_1(o(w_k))=s_2(o(w_k))$ due to the assumption of the lemma that $s_1(v)=s_2(v)$ for each view $v\in B$. Hence, the assumption $\pi\in MaxPath_{s_2}(A)$ implies that $\pi\in MaxPath_{s_1}(A)$ by Definition~\ref{path}. Therefore, $\pi\in Until(B,C)$ due to the assumption $MaxPath_{s_1}(A)\subseteq Until(B,C)$ of the lemma.

\noindent{\em Case II}: $o(w_k)\notin B\setminus C$ for some $k\ge 0$. Let $m\ge 0$ be the smallest such $k$ that $o(w_k)\notin B\setminus C$. Thus, 
\begin{enumerate}
    \item $o(w_k)\in (B\setminus C)$ for each $k< m$ and
    \item $o(w_{m})\notin (B\setminus C)$.
\end{enumerate}
Hence,
$o(w_k)\in B$ for each $k< m$. We now further split Case II into two different parts: 

\noindent{\em Part A}: $o(w_m)\in C$.
Therefore, $\pi\in Until(B,C)$ by Definition~\ref{until}. 

\noindent{\em Part B}: $o(w_m)\notin C$. Note that condition $o(w_k)\in B$ for each $k< m$ implies that $s_1(o(w_k))=s_2(o(w_k))$ for each $k<m$ due to the assumption of the lemma that $s_1(v)=s_2(v)$ for each $v\in B$. 
Thus, $w_0,w_1,\dots,w_m\in Path_{s_1}(A)$. By Lemma~\ref{maxpath exists}, this sequence can be extended to a sequence $\pi'=w_0,w_1,\dots,w_m,\dots\in MaxPath_{s_1}(A)$.

At the same time  $o(w_k)\in B$ for each integer $k\le m$ by the choice of $k$ and $o(w_m)\notin C$ by the assumption of the case. Thus, $\pi'\notin Until(B,C)$.
Therefore, $\pi' \in MaxPath_{s_1}(A)$ but $\pi'\notin Until(B,C)$, which contradicts the assumption of the lemma $MaxPath_{s_1}(A) \subseteq Until(B,C)$.
\end{proof}

We are now ready to finish the proof of the soundness of the remaining axioms of our logical system.

\begin{lemma}\label{transitivity sound}
If $T\vDash A\rhd_B C$  and $T\vDash C\rhd_D E$,  then $T\vDash A\rhd_{B\cup D} E$, where $B\cap D=\varnothing$.
\end{lemma}
\begin{proof}
By Definition~\ref{sat}, the assumption $T\vDash A\rhd_B C$ implies that there is a strategy $s_1$ such that $MaxPath_{s_1}(A)\subseteq Until(B,C)$.
Similarly,
the assumption $T\vDash C\rhd_D E$ implies that there is a strategy $s_2$ such that $MaxPath_{s_2}(C)\subseteq Until(D,E)$.

Define strategy $s$ as follows
\begin{equation}\label{s def}
s(v)=
\begin{cases}
s_1(v) & \mbox{ if } v\in B,\\
s_2(v) & \mbox{ otherwise}.
\end{cases}
\end{equation}
By Definition~\ref{sat}, it suffices to show that 
$MaxPath_s(A)\subseteq Until(B\cup D,E)$.
Indeed, consider any sequence $\pi=w_0,w_1,w_2,\dots\in MaxPath_s(A)$. We will show that $\pi\in Until(B\cup D,E)$. 

Note  that $s_1(v)=s(v)$ for each $v\in B$ by equation~(\ref{s def}). Also $MaxPath_{s_1}(A)\subseteq Until(B,C)$ by the choice of strategy $s_1$. Thus, $MaxPath_{s}(A)\subseteq Until(B,C)$ By Lemma~\ref{switch s}. Hence, $\pi\in  Until(B,C)$ because $\pi\in MaxPath_s(A)$. Thus, by Definition~\ref{until}, there is an integer $k_0\ge 0$ such that 
\begin{enumerate}
    \item $o(w_k)\in B$ for each integer $k$ such that $k< k_0$,
    \item $o(w_{k_0})\in C$.
\end{enumerate}
Note  that $s_2(v)=s(v)$ for each view $v\in D$ by equation~(\ref{s def})  and the assumption $B\cap D=\varnothing$. Also, $MaxPath_{s_2}(C)\subseteq Until(D,E)$ by the choice of strategy $s_2$. Thus, $MaxPath_{s}(C)\subseteq Until(D,E)$ by Lemma~\ref{switch s}.
Consider now the sequence $\pi'=w_{k_0},w_{k_0+1},w_{k_0+2},\dots\in MaxPath_s(C)$. Then, $\pi'\in Until(D,E)$. Hence, by Definition~\ref{until}, there is an integer $k'_0\ge k_0$ such that 
\begin{enumerate}
    \item $o(w_k)\in D$ for each integer $k$ such that $k_0\le k< k'_0$,
    \item $o(w_{k'_0})\in E$.
\end{enumerate}
Therefore, $\pi\in Until(B\cup D,E)$ by Definition~\ref{until}.
\end{proof}

\begin{lemma}
If $T\vDash A\rhd_B C$, then $T\vDash A\rhd_{B\setminus C} C$.
\end{lemma}
\begin{proof}
Let $T\vDash A\rhd_B C$.
Thus, $MaxPath_s(A)\subseteq Until(B,C)$ by Definition~\ref{sat}. Hence, 
$MaxPath_s(A)\subseteq Until(B,C)\subseteq Until(B\setminus C,C)$, by Lemma~\ref{until lemma}.
Therefore, $T\vDash A\rhd_{B\setminus C} C$ by Definition~\ref{sat}.
\end{proof}

\begin{lemma}
If $T\vDash A\rhd_\varnothing B$, then $T\vDash (A\setminus B) \rhd_\varnothing \varnothing$.
\end{lemma}
\begin{proof} By Definition~\ref{sat}, the assumption $T\vDash A\rhd_\varnothing B$ implies that there is a strategy $s$ such that $MaxPath_s(A)\subseteq Until(\varnothing, B)$. Thus, by Lemma~\ref{until lemma},
\begin{equation*}
    MaxPath_s(A\setminus B)\subseteq MaxPath_s(A)\subseteq Until(\varnothing,B).
\end{equation*}
At the same time, by Lemma~\ref{until lemma},
\begin{equation*}
    MaxPath_s(A\setminus B)\subseteq MaxPath_s(V\setminus B) \subseteq Until(\varnothing, V\setminus B).
\end{equation*}
Thus, by Lemma~\ref{until lemma},
\begin{eqnarray*}
MaxPath_s(A\setminus B)&\subseteq& Until(\varnothing,B) \cap Until(\varnothing, V\setminus B)\\
&=& Until(\varnothing,B\cap (V\setminus B))\\
&=& Until(\varnothing,\varnothing).
\end{eqnarray*}
Therefore, $T\vDash (A\setminus B) \rhd_\varnothing \varnothing$ by Definition~\ref{sat}.
\end{proof}

\begin{lemma}
If $T\vDash A\rhd_B \varnothing$, then $T\vDash A\rhd_\varnothing \varnothing$.
\end{lemma}
\begin{proof}
By Definition~\ref{sat}, the assumption $T\vDash A\rhd_B \varnothing$ implies that there is a strategy $s$ such that $MaxPath_s(A)\subseteq Until(B, \varnothing)$. Thus, by Lemma~\ref{until lemma},
$$
MaxPath_s(A)\subseteq Until(B,\varnothing)=\varnothing\subseteq Until(\varnothing,\varnothing).
$$
Therefore, $T\vDash A\rhd_\varnothing \varnothing$ by Definition~\ref{sat}.
\end{proof}

We end the section by stating the soundness theorem for our logical system. The theorem follows from the soundness of the individual axioms shown in the lemmas above.

\begin{theorem}
If $\vdash \phi$, then $T\vDash\phi$ for each epistemic transition system $T$.
\end{theorem}

\section{Completeness}\label{completeness section}

Suppose that set $X$ is a maximal consistent set of formulae in the language $\Phi$. In this section we define a canonical epistemic transition system $T(X)=(S,o,I,\{\to_i\}_{i\in I})$ based on set $X$.

\subsection{Valid Views}

Depending on the set $X$, there might be some views in set $V$ for which there are no states in the canonical epistemic transition system on which the observation function is equal to these views. Such views correspond to empty classes as discussed in Section~\ref{semantics section}. We intend to define the canonical epistemic transition system $T(X)$ in such a way that a view $v$ has no states in this system if and only if $X\vdash v\rhd_\varnothing \varnothing$. The rest of the views are referred to as {\em valid} views.

\begin{definition}\label{valid}
$Valid=\{v\in V\;|\; X\nvdash v\rhd_\varnothing\varnothing\}$.
\end{definition}

Below we prove several properties of valid views that are used later in the proof of the completeness.

\begin{lemma}\label{little induction}
$X\vdash \{b_1,\dots,b_n\}\rhd_\varnothing\varnothing$,
for each integer $n\ge 0$ and all views $b_1,\dots,b_n\in V\setminus Valid$.
\end{lemma}
\begin{proof}
We prove this lemma by induction on $n$. In the case $n=0$, we need to show that $X\vdash \varnothing \rhd_\varnothing \varnothing$, which is true by the Reflexivity axiom.

Suppose that $X\vdash\{b_1,\dots,b_{n-1}\}\rhd_\varnothing \varnothing$. Thus, by the Augmentation axiom, $X\vdash\{b_1,\dots,b_{n-1},b_n\}\rhd_\varnothing b_n$. On the other hand, the assumption $b_n\in V\setminus Valid$ by Definition~\ref{valid} implies $X\vdash b_n\rhd_\varnothing \varnothing$. Therefore, $X\vdash\{b_1,\dots,b_{n-1},b_n\}\rhd_\varnothing \varnothing$ by the Transitivity axiom.
\end{proof}

\begin{lemma}\label{cap valid}
$X\vdash A\rhd_B C\to A\rhd_{B\cap Valid} C$.
\end{lemma}
\begin{proof}
Lemma~\ref{little induction} implies that $X\vdash (B\setminus Valid) \rhd_\varnothing\varnothing$. Hence,  by Lemma~\ref{remove void}, 
$$X\vdash A\rhd_B C\to A\rhd_{B\setminus (B\setminus Valid)}C.$$ In other words, $X\vdash A\rhd_B C\to A\rhd_{B\cap Valid}C$.
\end{proof}

\begin{lemma}\label{cap valid right}
$X\vdash A\rhd_B C\to A\rhd_{B} (C \cap Valid)$.
\end{lemma}
\begin{proof}
Lemma~\ref{little induction} implies that $X\vdash (C\setminus Valid) \rhd_\varnothing\varnothing$. Thus, by the Augmentation axiom, $X\vdash C \rhd_\varnothing(C\cap Valid)$. At the same time,
$$
A\rhd_B C \to(C \rhd_\varnothing(C\cap Valid)\to A\rhd_{B} (C \cap Valid))
$$
is an instance of the Transitivity axiom. Therefore,  by the laws of propositional reasoning, $X\vdash A\rhd_B C\to A\rhd_{B} (C \cap Valid)$.
\end{proof}

\begin{lemma}\label{A' is B}
If $X\vdash A\rhd_B C$, then $(A\setminus C)\cap Valid  \subseteq B$.
\end{lemma}
\begin{proof}
Suppose that there is $v\in (A\setminus C)\cap Valid$ such that $v\notin B$. Thus, $v\in A$, $v\notin C$, $v\in Void$, and $v\notin B$. 

Recall that $X\vdash A\rhd_B C$ by the assumption of the lemma.  Hence, $X\vdash v\rhd_B C$ by Lemma~\ref{remove left} and due to $v\in A$. Thus, $X\vdash v\rhd_B B\cup C$ by Lemma~\ref{add right}. Then, $X\vdash v\rhd_{B\setminus (B\cup C)} B\cup C$ by the Early Bird axiom. In other words, $X\vdash v\rhd_{\varnothing} B\cup C$. Thus, $X\vdash v\rhd_{\varnothing} \varnothing$ by the Trivial Path axiom. Hence, $v\notin Valid$ by Definition~\ref{valid}, which contradicts the choice of view $v$.
\end{proof}

\subsection{Instructions}

For each formula $A\rhd_B C\in X$ our canonical epistemic transition system $T(X)$  will have a strategy to navigate from set $A$ to set $C$ through set $B$. Generally speaking, this strategy will use a dedicated instruction associated with formula $A\rhd_B C$. Formally, the set of all instructions is defined as a set of triples $(A,B,C)$ satisfying the three properties listed below:

\begin{definition}\label{canonical I}
Let $I$ be the set of all triples $(A,B,C)$ such that
\begin{enumerate}
    \item $A,B,C\subseteq Valid$,
    \item $X\vdash A\rhd_{A\cup B} C$,
    \item sets $A$, $B$, and $C$ are pairwise disjoint.
\end{enumerate}
\end{definition}

Note that technically instruction $(A,B,C)$ is associated not with formula $A\rhd_B C$, but rather with formula $A\rhd_{A\cup B} C$. This is done in order to be able to assume that sets $A$, $B$, and $C$ are pairwise disjoint. The next lemma is an general property of sets, which is used later.

\begin{lemma}\label{set lemma}
$(A\cup B)\setminus C= (A\setminus C)\cup (B\setminus (A\cup C))$.
\end{lemma}
\begin{proof}
\begin{eqnarray*}
(A\cup B)\setminus C &=& (A\cup (B\setminus A))\setminus C =
(A\setminus C)\cup ((B\setminus A)\setminus C)\\
&=&(A\setminus C)\cup(B\setminus (A\cup C))
\end{eqnarray*}
\end{proof}

We stated earlier that for each $A\rhd_B C\in X$ there is a dedicated instruction used by the strategy that navigates from set $A$ to set $C$ through set $B$. In some situations this dedicated instruction could be $(A,B,C)$. However in most cases we would need to slightly modify tuple $(A,B,C)$ into tuple $(A',B',C')$ in order for it to satisfy the three conditions from Definition~\ref{canonical I}. The next lemma specifies $(A',B',C')$ in terms of $(A,B,C)$ and proves that $(A',B',C')$ is an instruction.

\begin{lemma}\label{define prime}
If $X\vdash A\rhd_B C$, then $(A',B',C')\in I$, where $A'=(A\setminus C)\cap Valid$, $B'=(B\setminus (A\cup C))\cap Valid$, and $C'=C\cap Valid$.
\end{lemma}
\begin{proof}
By Definition~\ref{canonical I}, it suffices to show that sets $A'$, $B'$, and $C'$ are pairwise disjoint and that $X\vdash A'\rhd_{A'\cup B'} C'$. First, we show that these sets are pairwise disjoint:
\begin{eqnarray*}
A'\cap B'&=&[(A\setminus C)\cap Valid]\cap B'\subseteq A\cap B'\\
&=& A\cap [(B\setminus (A\cup C))\cap Valid]\subseteq A\cap [B\setminus A]=\varnothing,\\[2mm]
A'\cap C'&=&[(A\setminus C)\cap Valid]\cap C'\subseteq [A\setminus C]\cap C'\\
&=& [A\setminus C]\cap [C\cap Valid]\subseteq [A\setminus C]\cap C=\varnothing,\\[2mm]
B'\cap C'&=& [(B\setminus (A\cup C))\cap Valid]\cap C'\subseteq [B\setminus C] \cap C'\\
&=& [B\setminus C] \cap [C\cap Valid] \subseteq [B\setminus C] \cap C =\varnothing.
\end{eqnarray*}
Next, we show that $X\vdash A'\rhd_{A'\cup B'} C'$. Indeed, the assumption $X\vdash A\rhd_B C$ implies $X\vdash A\setminus C\rhd_B C$ by Lemma~\ref{remove left}. Hence, $X\vdash A\setminus C\rhd_{A\cup B} C$ by Lemma~\ref{add down}. Thus, $X\vdash A\setminus C\rhd_{(A\cup B)\setminus C} C$ by the Early Bird axiom. Then, by Lemma~\ref{set lemma}, 
$$X\vdash A\setminus C\rhd_{(A\setminus C)\cup (B\setminus (A\cup C))} C.$$ 
Thus, by Lemma~\ref{remove left},
$$X\vdash (A\setminus C)\cap Valid\rhd_{(A\setminus C)\cup (B\setminus (A\cup C))} C.$$
Hence, by Lemma~\ref{cap valid},
$$X\vdash (A\setminus C)\cap Valid\rhd_{((A\setminus C)\cup (B\setminus (A\cup C)))\cap Valid} C.$$
Thus, by Lemma~\ref{cap valid right},
$$X\vdash (A\setminus C)\cap Valid\rhd_{((A\setminus C)\cup (B\setminus (A\cup C)))\cap Valid} C\cap Valid.$$
Therefore, $X\vdash A'\rhd_{A'\cup B'} C'$ by the choice of sets $A'$, $B'$, and $C'$.
\end{proof}

Informally, the next lemma states that if there is an instruction to navigate from a set $A$ to an empty set, then set $A$ must be empty.

\begin{lemma}\label{C empty A empty}
For any $(A,B,C)\in I$ if $C=\varnothing$, then $A=\varnothing$.
\end{lemma}
\begin{proof}
The assumption $(A,B,C)\in I$ implies that $X\vdash A\rhd_{A\cup B} C$, by Definition~\ref{canonical I}. Thus, $X\vdash A\rhd_{A\cup B} \varnothing$ due to the assumption $C=\varnothing$. Hence, $X\vdash A\rhd_\varnothing \varnothing$ by the Path to Nowhere axiom. 
Suppose that there is a view $a\in A$. Hence,  $X\vdash a\rhd_\varnothing \varnothing$ by Lemma~\ref{remove left}. Thus, $a\notin Valid$ by Definition~\ref{valid}. At the same time, $A\subseteq Valid$ by Definition~\ref{canonical I}. Hence, $a\notin A$, which is a contradiction with the choice of view $a$. 
\end{proof}

\subsection{States and Observation Function}\label{canonical states and observations}

There are two types of states in the canonical epistemic transition system $T(X)$. The first type of states comes from our intention for each $v\in Valid$ to have at least one state $w$ such that $o(w)=v$. Thus, we consider each $v\in Valid$ to be a state of the first type and define the observation function on the states of the first type as $o(v)=v$.

In addition to the states of the first type, the canonical epistemic transition system also has states of the second type. Informally, these are intermediate states representing the result of a partial execution of an instruction. If an instruction $i$ might transition the system from a state $w$ of the first type  to a state $v$ of the first type, then the same instruction also might transition the system into a partial completion state $(u,i)$ of the second type. If the same instruction $i$ is invoked in state $(u,i)$, then the system will finish the transition into state $v$. If an instruction $j\neq i$ is invoked in state  $(u,i)$, then the system abandons the partially completed instruction $i$ and goes into a state prescribed by instruction $j$. The next two definitions formally capture this intuition. Symbol $\sqcup$ represents disjoint union.

\begin{definition}\label{canonical states}
$S=Valid \sqcup (Valid\times I)$.
\end{definition}

\begin{definition}\label{canonical o}
$o(v,B)=v$.
$$
o(w)=
\begin{cases}
w, & \mbox{ if $w\in Valid$},\\
v, & \mbox{ if $w=(v,i)$}.
\end{cases}
$$
\end{definition}

\begin{lemma}\label{o is valid}
$o(w)\in Valid$ for each $w\in S$.
\end{lemma}
\begin{proof}
The statement of the lemma follows from Definition~\ref{canonical states} and Definition~\ref{canonical o}.
\end{proof}

\subsection{Transitions}

Recall that the set of states of a canonical transition model is equal to the disjoint union $Valid \sqcup (Valid\times I)$. We refer to a state as having type one if it belongs to set $Void$ and type two if it belongs to set $Valid\times I$. 

\begin{figure}[ht]
\begin{center}
\scalebox{.7}{\includegraphics{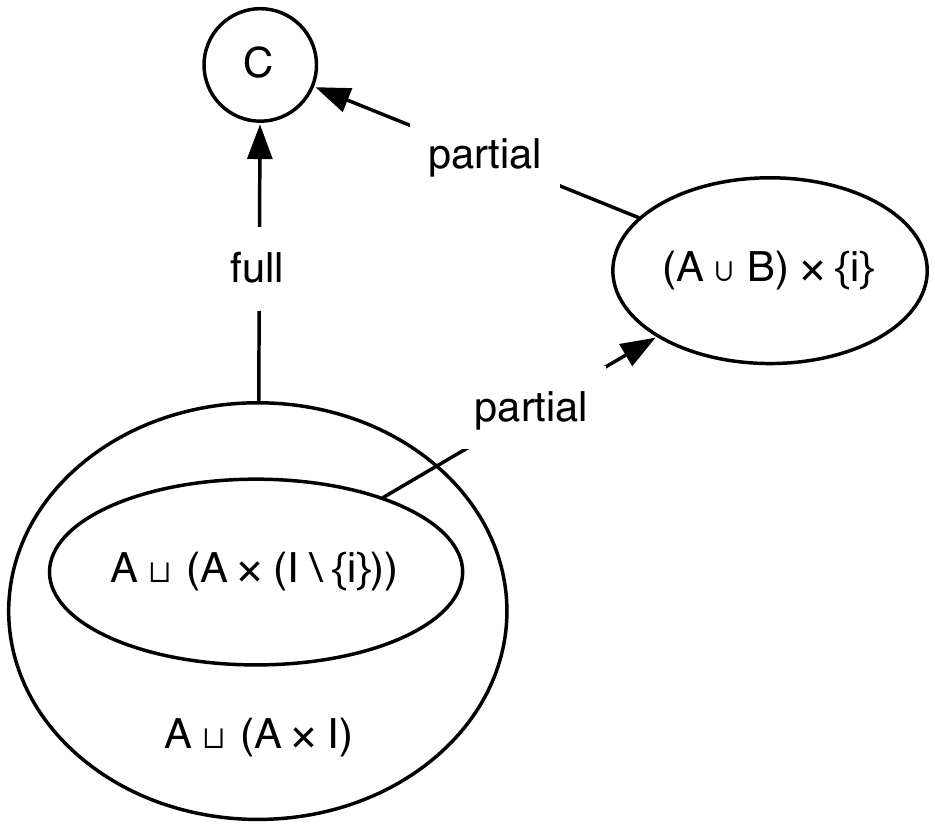}}
\caption{Transitions on instruction $(A,B,C)$}\label{transition figure}
\end{center}
\end{figure}

An instruction $(A,B,C)$ could be used to make one of the following transitions, see Figure~\ref{transition figure}:
\begin{enumerate}
    \item a ``full" transition from any state $w$ such that $o(w)\in A$ to any state $u\in C$ of the first type,
    \item a ``partial" transition from any state $w$ such that $o(w)\in A$ and state $w$ is not a partial completion for transition $(A,B,C)$, to a state $(v,(A,B,C))$ of the second type such that $v\in A\cup B$,
    \item a ``partial" transition from any state $w$ such that $o(w)\in A\cup B$ and  state $w$ is a partial completion for transition $(A,B,C)$, to any state $u\in C$.
\end{enumerate}

The next definition captures the above informal description.

\begin{definition}\label{canonical to}
If $i=(A,B,C)\in I$, then
 \begin{eqnarray*}
\to_{i}&=&\{(a,c)\;|\;a\in A\sqcup (A\times I), c\in C\}\\
            && \cup\;\{(a,b)\;|\;a\in A\sqcup (A\times (I\setminus\{i\})),b\in (A\cup B)\times\{i\}\}\\
            && \cup\; \{(b,c)\;|\;b\in (A\cup B)\times \{i\}, c\in C \}.
\end{eqnarray*}  
\end{definition}

This concludes the definition of the canonical epistemic transition system $T(X)=(S,o,I,\{\to_i\}_{i\in I})$. 
The next two lemmas prove basic properties of the transition relation $\to_i$. These properties are used later in the proof of the completeness.

\begin{lemma}\label{A is not last}
For any strategy $s$, any set $E\subseteq Valid$, any sequence $\pi\in MaxPath_s(E)$, and any element $w$ of $\pi$, if $s(o(w))=(A,B,C)$ and $o(w)\in A$, then $w$ cannot be the last element of sequence $\pi$.
\end{lemma}
\begin{proof}
By Definition~\ref{canonical o}, the assumption $o(w)\in A$ implies that $w\in A\sqcup (A\times I)$. By Lemma~\ref{C empty A empty}, the same assumption $o(w)\in A$ implies that there is a view $c\in C$. Thus, $w\to_{(A,B,C)}c$ by Definition~\ref{canonical to}. Hence, $w\to_{s(o(w))}c$ due to the assumption $s(o(w))=(A,B,C)$.  Therefore, element $w$ cannot be the last element of sequence $\pi$ by Definition~\ref{maxpath}.
\end{proof}


\begin{lemma}\label{AB is not last}
For any strategy $s$, any set $E\subseteq Valid$, any sequence $\pi\in MaxPath_s(E)$, and any two consecutive elements $w$ and $w'$ of $\pi$ such that 
\begin{enumerate}
    \item $w\in A\sqcup (A \times I)$, 
    \item $w'\in (A\cup B)\times\{(A,B,C)\}$,
    \item $s(o(w'))=(A,B,C)$,
\end{enumerate}
 the sequence $\pi$ contains an element $w''$ immediately after the element $w'$ such that $o(w'')\in C$.
\end{lemma}
\begin{proof}
By Definition~\ref{canonical o}, assumption $w\in A\sqcup (A \times I)$ implies that $o(w)\in A$. Hence, by Lemma~\ref{C empty A empty}, set $C$ contains at least one element $c$. Note that $w'\to_{(A,B,C)} c$ by Definition~\ref{canonical to} due to the assumption $w'\in (A\cup B)\times\{(A,B,C)\}$ of the lemma. Thus, by Definition~\ref{maxpath}, element $w'$ is not the last element of sequence $\pi$. 

Let $w''$ be the element of sequence $\pi$ that immediately follows element $w'$. Hence, $w'\to_{s(o(w'))} w''$. Thus,  $w'\to_{(A,B,C)} w''$ by the assumption $s(o(w'))=(A,B,C)$ of the lemma. Then, $w''\in C$ by Definition~\ref{canonical to} and due to the assumption $w'\in (A\cup B)\times\{(A,B,C)\}$ of the lemma. Therefore, $o(w'')\in C$ by Definition~\ref{canonical o}.
\end{proof}

\subsection{Provability Implies Satisfiability}

In this section we show if an atomic proposition is provable from set $X$, then it is satisfied in the canonical epistemic transition system. The converse of this statement is shown later in Lemma~\ref{true implies provable}.

\begin{lemma}\label{provable implies true}
If $X\vdash A\rhd_B C$, then $T(X)\vDash A\rhd_B C$.
\end{lemma}
\begin{proof}
Let $i_0=(A',B',C')$, where  $A'=(A\setminus C)\cap Valid$, $B'=(B\setminus (A\cup C))\cap Valid$, and $C'=C\cap Valid$. Thus, $i_0\in I$ by Lemma~\ref{define prime}. Define strategy $s$ to be a constant function such that $s(v)=i_0$ for each view $v\in V$. By Definition~\ref{sat}, it suffices to show that $MaxPath_s(A)\subseteq Until(B,C)$.

Consider any path $\pi=w_0,\dots\in MaxPath_s(A)$. By Definition~\ref{path}, $o(w_0)\in A$. 
Note that if $o(w_0)\in C$, then $\pi\in Until(B,C)$ by Definition~\ref{until}. In the rest of the proof, we assume that $o(w_0)\notin C$. Thus, $o(w_0)\in (A\setminus C)\cap Valid$ by Lemma~\ref{o is valid}. Hence, $o(w_0)\in A'$ by the choice of set $A'$. Then, by Lemma~\ref{A is not last}, sequence $\pi$ must contain at least one more element $w_1$ after element $w_0$. Then, $w_0\to_{s(o(w_0))}w_1$ by Definition~\ref{path}. Hence, $w_0\to_{i_0}w_1$ by the choice of strategy $s$. Then, $w_0\to_{(A',B',C')}w_1$  by the choice of instruction $i_0$. By Definition~\ref{canonical to}, statement $w_0\to_{(A',B',C')}w_1$ implies that one of the following three cases takes place:

\noindent{\em Case I:} $w_0\in A'\sqcup (A'\times I)$ and $w_1\in C'$. Thus, $o(w_0)\in A'$ and $o(w_1)\in C'$ by Definition~\ref{canonical o}. Hence, $o(w_0)\in (A\setminus C)\cap Valid$  and $o(w_1)\in C\cap Valid$ due to the choice of sets $A'$ and $C'$. Thus, $o(w_0)\in B$ by Lemma~\ref{A' is B} and also $o(w_1)\in C$. Therefore, $\pi\in Until(B,C)$ by Definition~\ref{until}.

\noindent{\em Case II:} 
$w_0\in A'\sqcup(A'\times (I\setminus\{i_0\}))$ and 
$w_1\in (A'\cup B')\times\{i_0\}$. Thus, $o(w_0)\in A'$ and $o(w_1)\in A'\cup B'$ by Definition~\ref{canonical o}. Hence, $o(w_0)\in (A\setminus C)\cap Valid$ and $o(w_1)\in ((A\setminus C)\cap Valid)\cup B'$ by the choice of set $A'$. Hence, $o(w_0)\in B$ and $o(w_1)\in B\cup B'$ by Lemma Lemma~\ref{A' is B}. Thus, $o(w_0),o(w_1)\in B$ by the choice of set $B'$.

Recall that $w_0\in A'\sqcup(A'\times (I\setminus\{i_0\}))$, $w_1\in (A'\cup B')\times\{i_0\}$, and $i_0=(A',B',C')$. Thus, by Lemma~\ref{AB is not last}, sequence $\pi$ must contain an element $w_2$ immediately after the element $w_1$ such that $o(w_2)\in C'$.  Hence $o(w_2)\in C$ by the choice of set $C'$. Thus, we have $o(w_0),o(w_1)\in B$ and $o(w_2)\in C$. Therefore, $\pi\in Until(B,C)$ by Definition~\ref{until}.

\noindent{\em Case III:}  $w_0\in (A'\cup B')\times \{i_0\}$ and 
$w_1\in C'$. Thus, $o(w_0)\in A'\cup B'$ and $o(w_1)\in C'$ by Definition~\ref{canonical o}. Hence, $o(w_0)\in (((A\setminus C)\cap Valid))\cup B'$ and $o(w_1)\in C$ by choice of sets $A'$ and $C'$. Thus, $o(w_0)\in B\cup B'$ by Lemma~\ref{A' is B} and also $o(w_1)\in C$. Hence, $o(w_0)\in B$ by the choice of set $B'$. Therefore, $\pi\in Until(B,C)$ by Definition~\ref{until}.
\end{proof}

\subsection{Satisfiability Implies Provability}

The goal of this section is to show the converse of Lemma~\ref{provable implies true}. This result is stated later in the section as Lemma~\ref{true implies provable}. To prove the result, due to Definition~\ref{sat}, it suffices to show that  $X\nvdash E\rhd_F G$ implies that $MaxPath_s(E)\nsubseteq Until(F,G)$ for each strategy $s$. In other words, we need to show that for any strategy $s$ there is a path $\pi\in MaxPath_s(E)$ that either never comes to $G$ or leaves $F$ before coming to $G$. To construct this path, we first define $G_*$ as a set of all starting states from which paths under strategy $s$ unavoidably lead to set $G$ never leaving set $F$. According to Definition~\ref{G*}, set $G_*$ is a union of an infinite chain of sets $G=G_0\subseteq G_1\subseteq G_2\dots$. Sets $\{G_i\}_{i\ge 0}$ are defined recursively below. The same definition also specifies the auxiliary families of sets $\{H_i\}_{i\ge 0}$, $\{A_i\}_{i\ge 1}$, $\{B_i\}_{i\ge 1}$, $\{C_i\}_{i\ge 1}$, $\{A^+_i\}_{i\ge 1}$, and $\{B^+_i\}_{i\ge 1}$ that will be used to state and prove various properties of family $\{G_i\}_{i\ge 0}$.

\begin{definition}\label{Gn} For any sets $F,G\subseteq V$ and any strategy $s$, let 
\begin{enumerate}
    \item\label{1} $G_0=G$ and $H_0=\varnothing$,
    \item choose any instruction $(A_n,B_n,C_n)\in I$ such that
        \begin{enumerate}
            \item\label{a} $A_n\cup B_n\subseteq F\cup G$,
            \item\label{b} $\{a\in A_n\;|\; s(a) = (A_n,B_n,C_n)\}\setminus G_{n-1}$ is not empty,
            \item\label{c} $\{a\in A_n\;|\; s(a) \neq (A_n,B_n,C_n)\}\subseteq G_{n-1}$,
            \item\label{d} $\{b\in B_n\;|\; s(b)\neq (A_n,B_n,C_n)\}\subseteq G_{n-1}$,
            \item\label{e} $C_n\subseteq G_{n-1}$
            
        \end{enumerate}
        and define
        \begin{enumerate}
            \item\label{i} $A_n^+=\{a\in A_n\;|\; s(a) = (A_n,B_n,C_n)\}$,
            \item\label{ii} $B_n^+=\{b\in B_n\;|\; s(b) = (A_n,B_n,C_n)\}$,
            \item\label{iii} $G_n=A^+_n \cup G_{n-1}$,
            \item\label{iv} $H_n=B^+_n \cup H_{n-1}$.
        \end{enumerate}
\end{enumerate}
\end{definition}

Next we state and prove properties of the families of the sets specified in Definition~\ref{Gn}.

\begin{lemma}\label{nm lemma}
$A_n\neq A_m$ for each $n > m$.
\end{lemma}
\begin{proof}
By item~(\ref{b}) of Definition~\ref{Gn}, there must exist a view $a_0\in  \{a\in A_n\;|\; s(a) = (A_n,B_n,C_n)\}$ such that $a_0\notin G_{n-1}$. Thus, $a_0\in A^+_n\setminus G_{n-1}$ by item~(\ref{i}) of Definition~\ref{Gn}.

Suppose that $A_n=A_m$. Hence, $A^+_n=A^+_m$ by item~(\ref{i}) of Definition~\ref{Gn}. Notice also that $G_{n-1}\supseteq G_{m}$ by item~(\ref{iii}) of Definition~\ref{Gn} and the assumption $n > m$. Then, $ a_0\in A^+_n\setminus G_{n-1}=A^+_m\setminus G_{n-1}\subseteq A^+_m\setminus G_{m}$. Hence, $A^+_m\nsubseteq G_m$, which contradicts item~(\ref{iii}) of Definition~\ref{Gn}.
\end{proof}

\begin{lemma}\label{ah lemma}
Sets $A^+_n$ and $H_{n-1}$ are disjoint.
\end{lemma}
\begin{proof}
Suppose that there is a view $v$ such that $v\in A^+_n$ and $v\in H_{n-1}$. Hence, by items~(\ref{1}) and (\ref{iv}) of Definition~\ref{Gn}, there must exist $m<n$ such that $v\in B^+_m$. Thus, $s(v)=(A_n,B_n,C_n)$ and $s(v)=(A_m,B_m,C_m)$ by items~(\ref{i}) and (\ref{ii}) of Definition~\ref{Gn}, which contradicts~Lemma~\ref{nm lemma}. 
\end{proof}

\begin{lemma}\label{bh lemma}
Sets $B^+_n$ and $H_{n-1}$ are disjoint.
\end{lemma}
\begin{proof}
Suppose that there is a view $v$ such that $v\in B^+_n$ and $v\in H_{n-1}$. Hence, by items (\ref{1}) and (\ref{iv}) of Definition~\ref{Gn}, there must exist $m<n$ such that $v\in B^+_m$. Thus, $s(v)=(A_n,B_n,C_n)$ and $s(v)=(A_m,B_m,C_m)$ by item (\ref{ii}) of Definition~\ref{Gn}, which contradicts~Lemma~\ref{nm lemma}. 
\end{proof}

\begin{lemma}\label{ab gh lemma}
$(A^+_n\cup B^+_n) \cap (G_{n-1}\cup  H_{n-1})\subseteq G_{n-1}$, for each $n\ge 1$.
\end{lemma}
By Lemma~\ref{ah lemma} and Lemma~\ref{bh lemma},
\begin{eqnarray*}
&&(A^+_n\cup B^+_n) \cap (G_{n-1}\cup  H_{n-1}) \\
&&= ((A^+_n\cup B^+_n) \cap G_{n-1}) \cup ( (A^+_n\cup B^+_n) \cap H_{n-1}) \\
&& = ((A^+_n\cup B^+_n) \cap G_{n-1}) \cup (A^+_n \cap H_{n-1}) \cup( B^+_n\cap H_{n-1})\\
&& \subseteq G_{n-1}\cup \varnothing\cup\varnothing=G_{n-1}.
\end{eqnarray*}

\begin{lemma}\label{ggg base lemma}
$X\vdash G_n\rhd_{G_n\cup B^+_n} G_{n-1}$, for each $n\ge 1$.
\end{lemma}
\begin{proof}
Note that $C_n\subseteq G_{n-1}$ by item (\ref{e}) of Definition~\ref{Gn}. Thus, $\vdash C_n\rhd_\varnothing G_{n-1}$ by the Reflexivity axiom. Also, $X\vdash A_n\rhd_{A_n\cup B_n}C_n$ by Definition~\ref{canonical I}. Thus, by the Transitivity axiom, $X\vdash A_n\rhd_{A_n\cup B_n}G_{n-1}$.
Hence, 
$$X\vdash A_n\rhd_{(A_n\setminus A^+_n)\cup A^+_n\cup(B_n\setminus B^+_n)\cup B^+_n}G_{n-1}$$
because $A^+_n\subseteq A_n$ and  $B^+_n\subseteq B_n$ by item~(\ref{i}) and item~(\ref{ii}) of Definition~\ref{Gn}. 
Note that $A_n\setminus A^+_n\subseteq G_{n-1}$ and $B_n\setminus B^+_n\subseteq G_{n-1}$ by items (\ref{c}), (\ref{d}), (\ref{i}) and (\ref{ii}) of Definition~\ref{Gn}. Thus, by Lemma~\ref{add down},
$$X\vdash A_n\rhd_{A^+_n\cup G_{n-1}\cup B^+_n}G_{n-1}.$$
Hence,  by Lemma~\ref{remove left} and due to item (\ref{i}) of Definition~\ref{Gn},
$$X\vdash A^+_n\rhd_{A^+_n\cup G_{n-1}\cup B^+_n}G_{n-1}.$$
Thus,
$X\vdash A^+_n\cup G_{n-1}\rhd_{A^+_n\cup G_{n-1}\cup B^+_n}G_{n-1}$ by the Augmentation axiom.
Then, $X\vdash G_n\rhd_{G_n\cup B^+_n}G_{n-1}$ by by item (\ref{iii}) of Definition~\ref{Gn}.
\end{proof}

\begin{lemma}\label{main lemma}
$X\vdash G_n\rhd_{G_n\cup  H_n} G$
\end{lemma}
\begin{proof}
We prove this statement by induction on $n$. If $n=0$, then $G_n=G$ by item (\ref{1}) of Definition~\ref{Gn}. Therefore, $\vdash G_n\rhd_{G_n\cup  H_n} G$ by the Reflexivity axiom.

Suppose that $n>0$. Thus, $X\vdash G_n\rhd_{G_n\cup B^+_n} G_{n-1}$ by Lemma~\ref{ggg base lemma}. 
Thus, by the Early Bird axiom, 
$$X\vdash G_n\rhd_{(G_n\setminus G_{n-1})\cup (B^+_n\setminus G_{n-1})} G_{n-1}.$$
Then, by item (\ref{iii}) of Definition~\ref{Gn},
$$X\vdash G_n\rhd_{A^+_n\cup (B^+_n\setminus G_{n-1})} G_{n-1}.$$ 
Hence, by Lemma~\ref{add down},
$$X\vdash G_n\rhd_{A^+_n\cup B^+_n} G_{n-1}.$$
At the same time, by the induction hypothesis,
$$X\vdash G_{n-1}\rhd_{G_{n-1}\cup  H_{n-1}} G.$$
Thus, by Lemma~\ref{super transitivity} taking into account Lemma~\ref{ab gh lemma},
$$X\vdash G_n\rhd_{A^+_n\cup B^+_n\cup G_{n-1}\cup  H_{n-1}} G.$$ 
Therefore, $X\vdash G_n\rhd_{G_{n}\cup  H_{n}} G$ by items (\ref{iii}) and (\ref{iv}) of Definition~\ref{Gn}.
\end{proof}

\begin{definition}\label{G*}
$G_*=\bigcup_{n}G_n$.
\end{definition}

\begin{lemma}\label{*=n}
There is $n\ge 0$ such that $G_*=G_n$.
\end{lemma}
\begin{proof}
By Definition~\ref{Gn} and Definition~\ref{G*}, we have 
$G_0\subseteq G_1\subseteq G_2\dots \subseteq G_*\subseteq V$. Thus, the statement of the lemma follows from the assumption in Section~\ref{syntax section} that set $V$ is finite.
\end{proof}

\begin{lemma}\label{to FG}
$G_n\cup H_n\subseteq F\cup G$.
\end{lemma}
\begin{proof}
Consider any $n\ge 0$. Note that $A_n\cup B_n\subseteq F\cup G$ by line~(\ref{a}) of Definition~\ref{Gn}. Thus, $A^+_n\cup B^+_n\subseteq F\cup G$ by line~(\ref{i}) and line~(\ref{ii}) of Definition~\ref{Gn}. Hence, $A^+_n\cup B^+_n\subseteq F\cup G$ for all $n\ge 0$. Then, $G_n\cup H_n\subseteq F\cup G$ for all $n\ge 0$ by by line~(\ref{a}), line~(\ref{iii}), and line~(\ref{iv}) of Definition~\ref{Gn}.
\end{proof}

We are now ready to state and prove the main lemma of this section. The statement of this lemma is the contrapositive of Lemma~\ref{provable implies true}.

\begin{lemma}\label{true implies provable}
If $T(X)\vDash E\rhd_F G$, then $X\vdash E\rhd_F G$.
\end{lemma}
\begin{proof}
Suppose that $T(X)\vDash E\rhd_F G$. Thus, by Definition~\ref{sat}, there is a strategy $s$ such that $MaxPath_s(E)\subseteq Until(F,G)$.

Consider chain of sets $G_0\subseteq G_1\subseteq G_2\subseteq \dots$ and set $G_*$, as specified in Definition~\ref{Gn} and Definition~\ref{G*}, constructed based on sets $F$ and $G$ as well as strategy $s$. We consider the following two cases separately:

\noindent{\em Case I:} $E\cap Valid\subseteq G_*$. Thus, by the Reflexivity axiom 
\begin{equation}\label{case 1 eq}
    \vdash (E\cap Valid)\rhd_\varnothing G_*.
\end{equation}
At the same time, $X\vdash (E\setminus Valid)\rhd_\varnothing\varnothing$ by Lemma~\ref{little induction}. Hence, 
$X\vdash (E\setminus Valid)\cup (E\cap Valid)\rhd_\varnothing(E\cap Valid)$ by the Augmentation axiom. In other words,
$X\vdash E\rhd_\varnothing(E\cap Valid)$. This, together with statement~(\ref{case 1 eq}) by the Transitivity axiom implies that  $X\vdash E\rhd_\varnothing G_*$.
Thus, by Lemma~\ref{*=n}, there is $n\ge 0$ such that $X\vdash E\rhd_\varnothing G_n$.  Hence, $X\vdash E\rhd_{G_n\cup H_n}G$ by Lemma~\ref{main lemma} and the Transitivity axiom. Hence, $X\vdash E\rhd_{F\cup G}G$ by Lemma~\ref{to FG} and Lemma~\ref{add down}. Thus, $X\vdash E\rhd_{(F\cup G)\setminus G}G$ by the Early Bird axiom. Note that $(F\cup G)\setminus G\subseteq F$. Therefore, $X\vdash E\rhd_{F}G$ by Lemma~\ref{add down}.

\noindent{\em Case II:} there is $e\in (E\cap Valid)\setminus G_*$. Let 
$$
W=(V\setminus G_*)\cup\{ (w,i)\in (V\setminus G_*)\times I\;|\;s(o(w))\neq i\}.
$$
Let $\pi$ be a maximal (either finite or infinite) sequence $w_0, w_1, \dots$ of elements from set $W$ such that
\begin{enumerate}
    \item $w_0=e$,
    \item $w_i\to_{s(o(w_i))} w_{i+1}$ for all $i\ge 0$.\label{item 2}
\end{enumerate}
\begin{claim}
Sequence $\pi$ is finite.
\end{claim}

\noindent{\em Proof of Claim.}
If sequence $\pi$ is infinite then $\pi \in MaxPath_s(E)$ by Definition~\ref{maxpath}, $o(w_0)=o(e)\in (E\cap Valid)\setminus G_* \subseteq E$, and item (\ref{item 2}) above. At the same time $\pi\notin Until(F,G)$ by Definition~\ref{until} because $o(w_i)\in o(W)\subseteq V\setminus G_* \subseteq V\setminus G_0=V\setminus G$ for each $i\ge 0$. Thus, $MaxPath_s(E)\nsubseteq Until(F,G)$, which is a contradiction with the choice of strategy $s$. \qed

 Let $w_k$ be the last element of sequence $\pi$ and $(A,B,C)=s(o(w_k))$. By $pr_1$ and $pr_2$ we mean the first and the second projection of a pair.

\begin{claim}\label{claim 1.5}
If $w_k\in V\times I$, then $pr_2(w_k)\neq (A,B,C)$.
\end{claim} 
\noindent{\em Proof of Claim.}
Suppose that $pr_2(w)=(A,B,C)$. Then, by the choice of instruction $(A,B,C)$, we have  $pr_2(w_k)=s(o(w_k))$. Thus, $w_k\notin W$ by the choice of set $W$, which is a contradiction with the choice of sequence $\pi$. \qed

\begin{claim}\label{claim A}
$o(w_k)\in A$.
\end{claim}

\noindent{\em Proof of Claim.}
Suppose $o(w_k)\notin A$. First we show $\pi\in MaxPath_s(E)$. Assume $\pi\notin MaxPath_s(E)$. Thus, by Definition~\ref{maxpath} and because $o(w_0)\in E$, there must exist state $w_{k+1}\in S$ such that $w_k\to_{s(o(w_k))} w_{k+1}$. Hence, $w_k\to_{(A,B,C)} w_{k+1}$ by the choice of the instruction $(A,B,C)$. Thus, by Definition~\ref{canonical to}, the assumption $o(w_k) \notin A$ implies that $w_k\in (A\cup B)\times\{(A,B,C)\}$, which is a contradiction with Claim~\ref{claim 1.5}. Therefore, $\pi\in MaxPath_s(E)$.

Recall that $MaxPath_s(E)\subseteq Until(F,G)$ by the choice of strategy $s$. Hence, $\pi\in Until(F,G)$. Thus, by Definition~\ref{until}, there is $m\ge 0$ such that $o(w_m)\in G$. Hence, $o(w_m)\in G_0$ by Definition~\ref{Gn}. Thus, $o(w_m)\in G_*$ by Definition~\ref{G*}. Therefore, $w_m\notin W$ by the choice of $W$ and Definition~\ref{canonical o}, which is a contradiction with the choice of sequence $\pi$. \qed

\begin{claim}\label{claim CG}
$C\subseteq G_*$.
\end{claim} 
\noindent{\em Proof of Claim.}
Suppose that there is $c\in C$ such that $c\notin G_*$. Note that $o(w_k)\in A$ by Claim~\ref{claim A}. Thus, $w_k\to_{(A,B,C)}c$ by Definition~\ref{canonical to} and the assumption $c\in C$. At the same time, the assumption $c\notin G_*$ implies $c\in V\setminus G_*$. Which implies that $c\in W$ by the choice of set $W$. Hence, sequence $\pi$ can be extended by at least one more element, namely by state $c$, which is a contradiction with the choice of sequence $\pi$. \qed

\begin{claim}\label{ABFG claim}
$A\cup B\subseteq F\cup G$.
\end{claim} 
\noindent{\em Proof of Claim.}
Suppose that there is $x\in (A\cup B)\setminus (F\cup G)$. Recall that $o(w_k)\in A$. Thus, $w_k\in A\sqcup (A\times (I\setminus\{(A,B,C)\}))$ by Definition~\ref{canonical o} and Claim~\ref{claim 1.5}. Thus, $w_k\to_{(A,B,C)}(x,(A,B,C))$ by Definition~\ref{canonical to} and because $x\in A\cup B$. Let $\pi'=\pi,(x,(A,B,C))$. In other words, $\pi'$ is the extension of sequence $\pi$ by an additional element $(x,(A,B,C))$. Note that $\pi'\in Path_s(E)$ by the choice of sequence $\pi$ and because $w_k\to_{(A,B,C)}(x,(A,B,C))$. By Lemma~\ref{maxpath exists}, sequence $\pi'$ can be extended to a sequence $\pi''\in MaxPath_s(E)$. Thus, $\pi''\in Until(F,G)$ by the choice of strategy $s$.

At the same time, $w_1,\dots,w_k\in W$ by the choice of sequence $\pi$. Thus, we have $o(w_1),\dots,o(w_k)\notin G_*$ by the choice of set $W$. Then, $o(w_1),\dots,o(w_k)\notin G_0$ by Definition~\ref{G*}. Hence, $o(w_1),\dots,o(w_k)\notin G$ by Definition~\ref{Gn}. Recall that $x\in (A\cup B)\setminus (F\cup G)$. Thus, $o(x,(A,B,C))\notin F\cup G$. Then, $o(x,(A,B,C))\notin F$ and $o(w_1),\dots,o(w_k),o(x,(A,B,C))\notin G$.
Therefore, $\pi''\notin Until(F,G)$ by Definition~\ref{until}, which is a contradiction with the above observation $\pi''\in Until(F,G)$.  \qed

\begin{claim}\label{claim ABG}
$\{x\in A\cup B\;|\; s(x)\neq (A,B,C)\}\subseteq G_*$.
\end{claim}
\noindent{\em Proof of Claim.} 
Let there be $x\in A\cup B$ such that $s(x)\neq (A,B,C)$ and $x\notin G_*$. Recall that $o(w_k)\in A$. Thus, $w_k\in A\sqcup (A\times (I\setminus\{(A,B,C)\}))$ by Definition~\ref{canonical o} and Claim~\ref{claim 1.5}. Thus, $w_k\to_{(A,B,C)}(x,(A,B,C))$ by Definition~\ref{canonical to} and because $x\in A\cup B$.  

At the same time, $o(x,(A,B,C))=x\notin G_*$ by Definition~\ref{canonical o} and the assumption $x\notin G_*$. Hence, $(x,(A,B,C))\in W$ by the assumption $s(x)\neq (A,B,C)$ and the choice of $W$.

Therefore, sequence $\pi$ can be extended by at least one more element, namely by state $(x,(A,B,C))$, which is a contradiction with the choice of sequence $\pi$. \qed 

We are now ready to finish the proof of the lemma. Note that set $G_{n}\setminus G_{n-1}$ is not empty for each $n\ge 0$ by item (\ref{b}) of Definition~\ref{Gn}. Thus, the recursive construction of chain $G_0\subseteq G_1\subseteq G_2\dots$, as given in Definition~\ref{Gn}, must terminate due to set $V$ being finite.
Suppose that  the last element of the chain $G_0\subseteq G_1\subseteq G_2\dots$ is set $G_{k-1}$. To come to a contradiction, it suffices to show that at least one more set can be added to the chain $G_0\subseteq G_1\subseteq G_2\dots$ by choosing instruction $(A_n,B_n,C_n)$ to be $(A,B,C)$. To prove the latter, we need to show that instruction $(A,B,C)$ satisfies conditions (\ref{a}) through (\ref{e}) of Definition~\ref{Gn}. Indeed, condition (\ref{a}) is satisfied by Claim~\ref{ABFG claim}. Condition~(\ref{b}) is satisfied because $s(o(w_k))\in A$ by Claim~\ref{claim A} and $o(w_k)\notin G_{k-1}=G_*$ because $w_k\in W$ by Claim~\ref{claim 1.5}. Conditions (\ref{c}) and (\ref{d}) are satisfied by Claim~\ref{claim ABG}. Finally, condition (\ref{e}) is satisfied by Claim~\ref{claim CG}.
This concludes the proof of the lemma.
\end{proof}

\subsection{Completeness: The Final Steps}

In this section we use Lemma~\ref{provable implies true} and Lemma~\ref{true implies provable} to finish the proof of the completeness theorem. The completeness theorem itself is stated below as Theorem~\ref{completeness}.

\begin{lemma}\label{almost there}
$T(X)\vDash \phi$ if and only if $\phi\in X$.
\end{lemma}
\begin{proof}
Induction on the structural complexity of formula $\phi$. In the base case the statement of the lemma follows from Lemma~\ref{provable implies true} and Lemma~\ref{true implies provable}. The induction step follows from the maximality and the consistency of set $X$ in the standard way.
\end{proof}

\begin{theorem}\label{completeness}
If $T\vDash\phi$ for every epistemic transition system $T$, then $\vdash\phi$.
\end{theorem}
\begin{proof}
Suppose $\nvdash\phi$. Let $X$ be a maximal consistent set containing formula $\neg\phi$. Thus, $T(X)\vDash\neg\phi$ by Lemma~\ref{almost there}. Therefore, $T(X)\nvDash\phi$.
\end{proof}

\section{Conclusion}\label{conclusion section}

The main informal contribution of this article is the observation that unlike its unrestricted counterpart, the restricted navigability relation is transitive not only for recall strategies but also for amnesic strategies. The main technical result is the completeness theorem for a logical system capturing the properties of the restricted navigability. Our setting is significantly different from the one in an earlier work by Li and Wang~\cite{lw17icla}, where a navigation strategy is defined  not as a function on views but as a fixed sequence of instructions.  As a result, the logical system that we propose is also significantly different from the one introduced by Li and Wang.

\bibliographystyle{elsarticle-num}
\bibliography{sp} 

\end{document}